\newif\ifarxiv
\arxivtrue

\ifarxiv
   \documentclass[onecolumn]{article}
\else
    \documentclass[1p]{elsarticle}
    \journal{Information Sciences}
\fi

\usepackage{epsfig}
\usepackage{booktabs}
\usepackage{amssymb,amsmath,amsthm,multirow}
\usepackage{amsfonts,graphicx,subfigure}
\usepackage{float}
\usepackage{caption}
\usepackage{subcaption}
\usepackage{algorithm}
\usepackage{algpseudocode}
\usepackage{booktabs}
\usepackage{multirow}
\usepackage{rotating}
\usepackage{hyperref}
\newtheorem{theorem}{Theorem}
\newtheorem{cor}{Corollary}
\newtheorem{lemma}{Lemma}

\newtheorem{dfn}{Definition}

\newcommand{\paperabstract}{The $k$-means++ algorithm is commonly restarted multiple times to avoid poor local optima, yet the number of restarts is almost always chosen arbitrarily and applied uniformly regardless of data set difficulty. This undermines any comparison relying on such a choice and wastes computation on easy data sets while potentially under-serving hard ones. Here, we introduce the Good-Turing Restart Criterion (GTRC). This combines a Good-Turing estimate, a proven unconditional bound, and a confidence-based bound on the probability that a further restart would improve on the current result, stopping once this probability falls below a user-specified tolerance $\varepsilon$. Our experiments on 34 real-world data sets show that GTRC identifies the point beyond which further $k$-means++ restarts yield only negligible improvement, achieving a more favourable balance between the number of restarts used and clustering quality than three existing stopping rules for multistart local search and popular fixed restart counts. \textbf{Software: }\href{https://github.com/RCdeAmorim/Good-Turing-Restart-Criterion}{https://github.com/RCdeAmorim/Good-Turing-Restart-Criterion}.}

\begin{document}

\ifarxiv
    \title{An interpretable Good--Turing restart criterion for $k$-means++}
    \author{Renato Cordeiro de Amorim\thanks{School of Computer Science and Electronic Engineering, University of Essex, Colchester, UK. r.amorim@essex.ac.uk}} 
    \date{}
    \maketitle
    \begin{abstract}
    \paperabstract
    \end{abstract}
    \vspace{1em}
\fi

\section{Introduction}

Clustering algorithms follow the unsupervised learning framework, and by consequence do not require labelled samples to learn from. Hence, such algorithms have become one of the main tools in exploratory data analysis, and have been applied to a number of areas such as bioinformatics, cybersecurity, quantitative finance, and computer vision, among others~\cite{ran2023comprehensive,oyewole2023data,wu2022construction,huang2024deep,lopez2025biclustering,lin2022clustering}.

The $k$-means++ algorithm~\cite{arthur2007k} remains, arguably, the most popular clustering algorithm in use today~\cite{ikotun2023k}. Given a data set and a number of clusters $k$, $k$-means++ produces a partition of $\mathcal{X}$ by iteratively minimising a within-cluster sum of squares objective. A well-known weakness of $k$-means++ is that it is not guaranteed to reach the global minimum of its objective, as it is sensitive to the choice of initial centroids (a weakness inherited from $k$-means, see Section \ref{sec:clustering}). To mitigate this, it is common practice to restart $k$-means++ multiple times and select as final partition that with the lowest objective value.

This practice is so widespread that it has become the default behaviour in popular software packages such as scikit-learn~\cite{scikit-learn}. However, the number of such restarts is invariably chosen arbitrarily. Some studies use ten restarts, others twenty, and others still one hundred. And potentially even more problematic: after the number of restarts is selected, it is usually applied to all data sets regardless of their structure. When a fixed number of restarts is applied indiscriminately across data sets of varying complexity, the resulting baseline is either wasteful on easy data sets or inadequate on hard ones, potentially undermining the validity of any comparison that relies on it. 

To our knowledge, there is no well established method providing a principled basis for the choice of restarts. In this paper, we address this problem directly by introducing the Good--Turing restart criterion (GTRC), a theoretically-grounded, data-adaptive stopping rule with a clear probabilistic interpretation. It halts when the estimated probability that a further restart would improve on the best partition found so far falls below a user-specified tolerance. This makes the criterion not only adaptive to the complexity of each data set, but also transparent and interpretable. Our experiments on 34 real-world data sets show that GTRC identifies the point beyond which further $k$-means++ restarts yield only negligible improvement, achieving a more favourable balance between the number of restarts used and clustering quality than three existing stopping rules for multistart local search and popular fixed restart counts.

\section{Background}
\label{sec:background}

In this section we review the prior work this paper draws on. We first summarise the $k$-means and $k$-means++ algorithms and the role of restarts in mitigating their sensitivity to initialisation. We then review existing stopping rules for multistart local search. Finally we review Good--Turing estimation, the statistical tool underlying the restart criterion proposed in Section~\ref{sec:gt_kmeans}.

\subsection{Clustering}
\label{sec:clustering}
Let $\mathcal{X} = \{x_1, \ldots, x_n\} \subset \mathbb{R}^d$ be a data set and $k$ be the number of clusters, with $k < n$. The $k$-means algorithm \cite{macqueen1967some} generates a partition $\mathcal{C} = (C_1, \ldots, C_k)$ of $\mathcal{X}$ minimising
\begin{equation}
\label{eq:kmeans}
    \phi = \sum_{l=1}^k \sum_{x_i \in C_l} \sum_{v=1}^d (x_{iv} - \mu_{lv})^2, 
\end{equation}
where 
\[
\mu_{lv} = \frac{1}{|C_l|}\sum_{x_i \in C_l} x_{iv}.
\]
This minimisation follows three straightforward steps: (i) Select $k$ data points from $\mathcal{X}$, uniformly at random, and copy their values to \(\mu_1, \ldots, \mu_k\); (ii) for each \(x_i \in \mathcal{X}\), identify its nearest centroid \(\mu_l\) and assign \(x_i\) to \(C_l\); (iii) update each centroid \(\mu_l\) to the component-wise mean over all data points assigned to \(C_l\). Steps (ii) and (iii) are repeated until convergence.

Although very popular, $k$-means is not without drawbacks (for further details, see~\cite{ahmed2020k,liu2023transforming}). Here, we are particularly interested in the fact that $k$-means is not guaranteed to find the global minimum of~\eqref{eq:kmeans}, as the quality of the final partition, $\mathcal{C}$, is heavily influenced by the choice of initial centroids. The major issue is that these initial centroids are found at random.

The above led to a considerable amount of research focusing on finding good initial centroids for $k$-means, or at least increasing the probability of doing so~\cite{harris2022extensive,franti2019much, de2012empirical}. Among these, $k$-means++~\cite{arthur2007k} is certainly the most popular approach. $k$-means++ is the default initialisation in popular software packages such as scikit-learn, MATLAB, and R~\cite{scikit-learn,matlab, r_language} and has been shown both theoretically and empirically to produce better clusterings than random initialisation. Rather than selecting all $k$ initial centroids uniformly at random, $k$-means++ selects them sequentially. The first centroid $\mu_1$ is chosen uniformly at random from $\mathcal{X}$. Each subsequent centroid $\mu_l$, for $l = 2, \ldots, k$, is then chosen from $\mathcal{X}$ with probability proportional to
\begin{equation}
\label{eq:kmeanspp}
    P(x_i) = \frac{D(x_i)^2}{\sum_{x_j \in \mathcal{X}} D(x_j)^2},
\end{equation}
where $D(x_i)$ denotes the distance from $x_i$ to its nearest already-selected centroid. This encourages a spread of initial centroids 
across the data set, reducing the likelihood of a poor initialisation. Once all $k$ centroids have been selected, $k$-means++ proceeds 
identically to $k$-means.

\subsection{The restart problem}
\label{sec:restart_problem}

Despite the improvements offered by $k$-means++, there is still no guarantee it will find the global minimum of~\eqref{eq:kmeans}. It is therefore common practice to run $k$-means++ multiple times and retain the partition with the lowest value of $\phi$. The number of such restarts, however, is typically chosen arbitrarily and applied uniformly across all data sets under study, regardless of their complexity. As we shall show (see Section \ref{sec:results}), the number of restarts required may vary considerably across data sets, and a principled, data-adaptive criterion for determining when to stop is therefore both natural and necessary.

Perhaps the most systematic study of the effect of restarts on $k$-means performance is that of Fr\"{a}nti and Sieranoja~\cite{franti2019much}. They show that repeating $k$-means $100$ times reduces the proportion of erroneous clusters from approximately $15\%$ to $6\%$, on average, across their benchmark (and further down to 1\% when using the Maxmin initialisation). However, they also observe that this improvement is heavily data-dependent. On data sets with high cluster overlap, $k$-means converges reliably and few restarts suffice, whereas on data sets with well-separated clusters, even a large number of restarts may not suffice. Crucially, they acknowledge that the improvement probability $p$ is unknown in practice and so they could not automatically derive the number of restarts ($r$), leaving $r = 100$ without principled justification.

This lack of a principled stopping rule is pervasive across the broader literature. There are clear examples (including our own previous work) of extensive benchmarks of $k$-means variants and initialisation strategies across synthetic and real-world data sets adopting $r = 50$ restarts as the default for stochastic methods, without justification for this choice~\cite{vouros2021empirical,harris2022extensive}. Recent studies proposing new $k$-means variants suffer from the same issue, often comparing their new algorithm against $k$-means++, and others, using $10$ or $20$ restarts ~\cite{capo2020efficient}. Taken together, these studies illustrate not merely that different authors choose different restart counts, but that within any given study the same count is applied uniformly across data sets of very different complexity.

The consequences of the arbitrary choice of $r$ extend beyond inconsistency. When $k$-means++ with a fixed $r$ is used as a baseline in a comparative study, the outcome of the comparison is implicitly sensitive to that choice. On an easy data set, even a small $r$ suffices to find a near-optimal partition, and a proposed algorithm need only outperform a competent baseline. On a hard data set, an insufficient $r$ produces a weak baseline, making it easier for any competitor to appear superior (not because the competitor is genuinely better, but because the baseline was undersupplied with restarts). Conversely, an excessively large $r$ on an easy data set wastes computation without improving the baseline. In either case, the reader has no way to determine which regime applies, since the choice of $r$ is reported but never justified. We address this gap by proposing an interpretable, theoretically grounded criterion for determining when enough $k$-means++ restarts have been performed.

\subsection{Stopping rules for multiple restarts}
\label{sec:stopping_rules_review}

Stopping rules for multistart local search have been studied largely in the context of continuous, box-constrained optimisation rather than clustering. Some of this work pursues a different goal from the one considered here. For instance, Lagaris and Tsoulos~\cite{lagaris2008stopping} proposed three rules aimed at certifying that most or all local optima of a function have been recovered. These rules do not apply in our context, since we are only interested in finding local minima that are better than those already found. Gy\"{o}rgy and Kocsis~\cite{gyorgy2011efficient} proposed a strategy capable of running many $k$-means runs concurrently, deciding whether each run should compute its next iteration. These runs can be seen as different restarts, and their algorithm decides how much computing power each should take. However, the authors point out that a single $k$-means run usually converges in so few iterations that this strategy has little opportunity to make that choice before each run is already finished.

We now turn to three rules deciding when to stop restarts given the clustering results obtained so far. These span old and recent proposals, adaptive and non-adaptive designs, and distinct theoretical foundations, and form the basis of our empirical comparison in Section~\ref{sec:results}.

Hart's rule~\cite{hart1998sequential} bounds $P(\phi^{(r)} < \min(\phi^{(1)}, \ldots, \phi^{(r-1)}))$ using a count $\hat\rho_r(\epsilon)$ of prior restarts within a tolerance $\epsilon$ of the best value found so far, giving a probability estimate $\hat\rho_r(\epsilon)/r$, combined with a normal-approximation confidence term governed by a parameter $\delta$. This count is built recursively from the sequence of successive improvements in $\phi$. The rule stops at the first $r$ for which
\[
\Phi\!\left(2\delta\sqrt{r}\right) - \Phi\!\left(-2\delta\sqrt{r}\right) - \left(1 - \frac{\hat\rho_r(\epsilon)}{r}\right)^{r} \;\geq\; 1-\beta,
\]
where $\Phi$ is the standard normal cumulative distribution function, and $\beta$ is the maximum acceptable probability that the rule stops before observing an improvement.

Ohsaki and Yamakawa's rule~\cite{ohsaki2018stopping} estimates the likelihood ratio between the hypotheses that all local optima have been found and that one further, undiscovered optimum remains, inferring the size of its attractor from those already observed. Let $\mathcal{C}^{(1)}, \ldots, \mathcal{C}^{(r)}$ denote the partitions found by restarts $1, \ldots, r$, and let $w$ denote the number of distinct partitions among them, with $n_1, \ldots, n_w$ the number of restarts converging to each. The likelihood ratio reduces to
\[
\frac{\bar{L}^\star_w}{L^{(0)}_w} = \frac{1}{w}\sum_{j=1}^{w}\left(1 - \frac{n_j}{r+n_j}\right)^{r},
\]
and the rule stops once this ratio falls below a threshold. 

Corominas's rule~\cite{corominas2023deciding} does not monitor the run in progress; instead, it computes a fixed number of restarts $r$ in advance. This $r$ is chosen so that, if the true probability of a restart improving on the current best were at least $\epsilon_C$, the probability of observing no improvement in $r$ restarts would be at most $\alpha$,
\[
r = \left\lceil \frac{\ln \alpha}{\ln(1-\epsilon_C)} \right\rceil,
\]
where $\epsilon_C$ is an assumed threshold on the improvement probability itself, and $\alpha$ is a failure tolerance.

\subsection{Good--Turing estimation}
\label{sec:good_turing_estimation}

The Good--Turing frequency estimation~\cite{good1953} addresses a deceptively simple problem. Given a sample drawn from a population of unknown composition, estimate the probability that the next draw will reveal something never seen before. Let a sample of size $r$ contain $\phi_1$ outcomes observed exactly once. Good, building on an idea attributed to Turing, showed that the expected total probability mass attributable to distinct outcomes observed in the sample is approximately $1 - \phi_1/r$, and consequently that the probability the next draw reveals a previously unseen outcome (the missing mass) is approximately $\phi_1/r$. This estimator is remarkable in that it requires no parametric assumptions about the underlying population. Nonetheless, it is not without flaws. The point estimate $\phi_1/r$ is itself a random quantity, subject to sampling variability which may be particularly pronounced when $r$ is small~\cite{mcallester2000convergence}.

Several refinements of the Good--Turing estimator have since been proposed, combining it with empirical-frequency estimates to obtain near-optimal accuracy across an entire distribution~\cite{orlitsky2015competitive,acharya2013optimal}. However, these target a different problem than ours. They attempt to estimate the probability of every possible outcome. This does not apply to our case, as we are not interested in the probability of every local minimum, which would in any case require knowing how many such minima exist. 

A stopping rule for $k$-means++ restarts built solely on the Good--Turing estimator risks halting prematurely, due to its
variability when $r$ is small, which we discussed above. To address this, in Section~\ref{sec:gt_kmeans} we adapt this estimator, alongside other components, to bound the probability that restart $r+1$ of $k$-means++ produces a new and better value of $\phi$.

\section{The Good--Turing Restart Criterion}
\label{sec:gt_kmeans}

This section presents our main contribution, the GTRC. We first develop the theoretical results on which the criterion rests, and then combine them into a practical, interpretable stopping rule for $k$-means++ restarts.

\subsection{Theoretical Framework}
\label{sec:theoretical_framework}

Let $\phi^{(r)}$ denote the value of \eqref{eq:kmeans} returned by the $r$-th independent run of $k$-means++ on a fixed data set $\mathcal{X}$ and cluster count $k$. We first establish that this sequence has the distributional structure needed for the results that follow.

\begin{lemma}
\label{lemma:iid}
The $k$-means++ objectives  $\phi^{(1)}, \ldots, \phi^{(r)}$, produced by independent runs with fixed \(\mathcal{X}\) and \(k\) are mutually independent and identically distributed.
\end{lemma}
\begin{proof}
Let $M^{(i)}$ denote the set of initial centroids selected by run $i$. Since $\mathcal{X}$ and $k$ are fixed, every run draws $M^{(i)}$ from
the same distribution, determined by the $k$-means++ probability rule~\eqref{eq:kmeanspp}, and distinct runs draw independently of one another. Hence $M^{(1)}, \ldots, M^{(r)}$ are mutually independent and identically distributed.

Given its initial centroids, the $k$-means iteration is a deterministic procedure, so there exists a fixed function $g$, depending only on $\mathcal{X}$ and $k$, such that $\phi^{(i)} = g(M^{(i)})$ for every $i$. Applying the same fixed function to i.i.d.\ random variables preserves both mutual independence and the common distribution, so $\phi^{(1)}, \ldots, \phi^{(r)}$ are mutually independent and identically distributed.
\end{proof}

With the above established, we can formalize what it means for a restart to improve on the results obtained so far.

\begin{dfn}
Run $r$ is an \emph{improvement} if $\phi^{(r)} < \min(\phi^{(1)}, \ldots, \phi^{(r-1)})$ for $r \geq 2$. Run $1$ is an improvement by convention.
\end{dfn}

Let $m$ denote the number of distinct values of $\phi$ attainable upon convergence of $k$-means++ on $(\mathcal{X}, k)$. We now establish a lower bound on the probability of ties between runs, in terms of $m$.

\begin{lemma}
\label{lemma:ties_prob_is_positive}
For any two runs $i \neq j$, 
\[
P(\phi^{(i)} = \phi^{(j)}) \geq \frac{1}{m}.
\]
\end{lemma}
\begin{proof}
Let $p_s$ denote the probability that a run converges to the $s$-th local minimum, for $s = 1, \ldots, m$. Hence,
\begin{align*}
P(\phi^{(i)} = \phi^{(j)}) &= \sum_{s=1}^{m} P(\phi^{(i)} = \phi^{(j)} = s\text{-th minimum}) \\
& \stackrel{\text{Lemma \ref{lemma:iid}}}{=}\sum_{s=1}^{m} P(\phi^{(i)} = s\text{-th minimum}) P(\phi^{(j)} = s\text{-th minimum})=\sum_{s=1}^{m} p_s^2.
\end{align*}
Note that \(\sum_{s=1}^{m}\left(p_s - \frac{1}{m}\right)^2 \geq 0\), which we expand to
\[
\sum_{s=1}^{m} p_s^2 - \frac{2}{m}\sum_{s=1}^{m} p_s + \sum_{s=1}^{m}\frac{1}{m^2} \geq 0.
\]
Given $\sum_{s=1}^{m} p_s = 1$ and $\sum_{s=1}^{m}\frac{1}{m^2} 
= \frac{1}{m}$, the above leads to
\[
\sum_{s=1}^{m} p_s^2 - \frac{2}{m} + \frac{1}{m} \geq 0.
\]
Hence $P(\phi^{(i)} = \phi^{(j)}) = \sum_{s=1}^{m} p_s^2 \geq 1/m$, 
with equality if and only if $p_s = 1/m$ for all $s$.
\end{proof}

This pairwise bound extends naturally to the probability that all $r$ runs agree.

\begin{lemma}
\label{lemma:all_runs_equal}
For any $r \geq 2$,
\[
P(\phi^{(1)} = \ldots = \phi^{(r)}) \geq \left(\frac{1}{m}\right)^{r-1}.
\]
\end{lemma}
\begin{proof}
Since $\phi^{(1)}, \ldots, \phi^{(r)}$ are i.i.d., each run independently converges to some local minimum with probability $p_s$ for $s = 1, \ldots, m$, so
\[
P(\phi^{(1)} = \ldots = \phi^{(r)}) = \sum_{s=1}^{m} p_s^r = \sum_{s=1}^{m} p_s \cdot p_s^{r-1}.
\]
This is a weighted average of $p_s^{r-1}$ with weights $p_s$, where $\sum_{s=1}^{m} p_s = 1$. Since $f(x) = x^{r-1}$ is convex for $r \geq 2$, Jensen's inequality gives
\[
\sum_{s=1}^{m} p_s \cdot p_s^{r-1} \geq \left(\sum_{s=1}^{m} p_s \cdot p_s\right)^{r-1} = \left(\sum_{s=1}^{m} p_s^2\right)^{r-1} \geq \left(\frac{1}{m}\right)^{r-1},
\]
where the last inequality follows since $\sum_{s=1}^{m} p_s^2 = 
P(\phi^{(i)} = \phi^{(j)}) \geq \frac{1}{m}$ by 
Lemma~\ref{lemma:ties_prob_is_positive}.
\end{proof}

We are now ready to bound the probability that a given restart constitutes an improvement.

\begin{theorem}    
\label{thm:prob_of_improvement}
The probability that run $r \geq 2$ constitutes an improvement satisfies
\[
    P\!\left(\phi^{(r)} < \min\!\left(\phi^{(1)}, \ldots, \phi^{(r-1)}\right)\right) \leq \frac{1}{r}\left(1 - \left(\frac{1}{m}\right)^{r-1}\right).
\]
\end{theorem}
\begin{proof}
The events $\{\phi^{(i)}$ is the strict minimum of $\phi^{(1)}, \ldots, \phi^{(r)}\}$ for $i = 1, \ldots, r$ are mutually exclusive. By Lemma \ref{lemma:iid}, each has the same probability, so 
\[
r \cdot P\!\left(\phi^{(r)} < \min\!\left(\phi^{(1)}, \ldots, \phi^{(r-1)}\right)\right)
\]
 is the probability that exactly one of the \(r\) runs achieves the minimum. Its complement is the probability that at least two of the \(r\) runs tie at the minimum. Hence,
\[
r \cdot P\!\left(\phi^{(r)} < \min\!\left(\phi^{(1)}, \ldots, \phi^{(r-1)}\right)\right) = 1-P\!\left(\exists\, i \neq j : \phi^{(i)} = \phi^{(j)} = \min\!\left(\phi^{(1)}, \ldots, \phi^{(r)}\right)\right).
\]
If all $r$ runs return the same value, all $r$ runs tie at the minimum, so
\[
P\!\left(\exists\, i \neq j : \phi^{(i)} = \phi^{(j)} = \min\!\left(\phi^{(1)}, \ldots, \phi^{(r)}\right)\right) \geq P(\phi^{(1)} = \ldots = \phi^{(r)}) \geq \left(\frac{1}{m}\right)^{r-1},
\]
where the last inequality uses Lemma~\ref{lemma:all_runs_equal}. Hence,
\[
    P\!\left(\phi^{(r)} < \min\!\left(\phi^{(1)}, \ldots, \phi^{(r-1)}
    \right)\right) \leq \frac{1}{r}\left(1 - \left(\frac{1}{m}
    \right)^{r-1}\right).
\]
\end{proof}

This bound has an immediate consequence for the expected number of improving restarts.

\begin{cor}
\label{cor:n_r}
Let $N_r$ denote the number of improvements in $r$ independent runs, with $r \geq 2$. Then
\[
\mathbb{E}[N_r] \leq 1 + \sum_{i=2}^{r} \frac{1}{i}\left(1 - \left(\frac{1}{m}\right)^{i-1}\right).
\]
\end{cor}
\begin{proof}
We can write $N_r = \sum_{i=1}^r \mathbf{1}\!\left[\phi^{(i)} < \min\!\left(\phi^{(1)},\ldots,\phi^{(i-1)}\right)\right]$, where the term for $i=1$ equals $1$ by convention. By linearity of expectation and Theorem~\ref{thm:prob_of_improvement},
\[
\mathbb{E}[N_r] = 1 + \sum_{i=2}^{r} P\!\left(\phi^{(i)} < \min\!\left(\phi^{(1)},\ldots,\phi^{(i-1)}\right)\right) \leq  1 + \sum_{i=2}^{r} \frac{1}{i}\left(1 - \left(\frac{1}{m}\right)^{i-1}\right) . 
\]
\end{proof}

Since $\sum_{i=2}^{r} 1/i$ grows like $\ln r$, the bound in Corollary~\ref{cor:n_r} shows that the expected number of improving restarts grows only logarithmically in $r$, even though $r$ itself grows without bound. In practice, this means that after a handful of runs, additional restarts become increasingly unlikely to improve on the best partition found so far. This diminishing-returns behaviour is precisely what motivates a stopping rule: rather than fixing $r$ in advance, we would like to detect from the data itself the point at which further restarts are no longer worth their computational cost.

The bound in Theorem~\ref{thm:prob_of_improvement} also constrains how quickly improvements can accumulate over many runs. This has a direct practical use: it tells us how many runs must pass before another improvement can reasonably be expected, which we use to set a minimum number of runs before the stopping criterion is checked at all.

\begin{theorem}
\label{thm:improvement_expectation_lb}
Let us assume that the global minimum of \eqref{eq:kmeans} is not one of $\phi^{(1)}, \ldots, \phi^{(r)}$. Then, if $r'$ is the smallest integer greater than $r$ such that $\mathbb{E}[N_{r'}] - \mathbb{E}[N_r] \geq 1$, we have that  $r' >e \cdot r$.
\end{theorem}
\begin{proof}
By definition, 
\[
N_{r'} - N_r = \sum_{i=r+1}^{r'} \mathbf{1}\!\left[\phi^{(i)} < \min(\phi^{(1)},\ldots,\phi^{(i-1)})\right].
\]
By linearity of expectation and Theorem~\ref{thm:prob_of_improvement},
\[
\mathbb{E}[N_{r'}] - \mathbb{E}[N_r] = \sum_{i=r+1}^{r'} P\!\left(\phi^{(i)} < \min(\phi^{(1)},\ldots,\phi^{(i-1)})\right) 
\leq \sum_{i=r+1}^{r'} \frac{1}{i}\left(1-\left(\frac{1}{m} \right)^{i-1}\right).
\]
Since $(1/m)^{i-1} > 0$ we have that
\[
 1- \left(\frac{1}{m}\right)^{i-1} < 1 \Rightarrow \frac{1}{i}\left(1-\left(\frac{1}{m} \right)^{i-1}\right) <\frac{1}{i}.
\]
Summing over $i=r+1,\ldots,r'$, and noting that $\sum_{i=r+1}^{r'} \frac{1}{i} < \ln(r'/r)$, the sum is strictly less than $\ln(r'/r)$. Reaching $1$ therefore requires $\ln(r'/r) > 1$, thus $r' > e \cdot r$.
\end{proof}

\subsection{Algorithm}

We now turn the results of Section~\ref{sec:theoretical_framework} into a practical stopping criterion. Our criterion requires an upper bound on $P(\phi^{(r)} < \min(\phi^{(1)}, \ldots, \phi^{(r-1)}))$ that can be computed from the runs observed so far, small enough once improvement becomes unlikely, and reliable throughout. No single estimator meets all three requirements on its own. A purely statistical estimate can be noisy or misleading when few runs have been observed, while a bound that is valid unconditionally is typically far looser than necessary once more information becomes available. We therefore combine three distinct bounds on this probability, each valid under different circumstances, and take their minimum at every run. 

The first of the three bounds is the Good--Turing missing-mass estimate introduced in Section~\ref{sec:good_turing_estimation}. Since an improvement requires run $r$ to be strictly smaller than all previous runs, it must return a value of $\phi$ not seen before. Hence,
\[
P\!\left(\phi^{(r)} < \min\!\left(\phi^{(1)}, \ldots, \phi^{(r-1)}\right) \right) \leq P\!\left(\phi^{(r)} \notin \{\phi^{(i)} \mid i = 1, \ldots, r-1\}\right),
\]
where the right-hand side is the probability that the next run yields a previously unseen value of $\phi$. Writing $\phi_1$ for the number of distinct values of $\phi$ that appear exactly once among $\phi^{(1)}, \ldots, \phi^{(r)}$, the Good--Turing estimate of the missing mass after $r$ runs is
\[
\widehat{P}\!\left(\phi^{(r)} \notin \{\phi^{(i)}\}\right) = \frac{\phi_1}{r}.
\]
As discussed in Section~\ref{sec:good_turing_estimation}, this estimate is subject to sampling variability that can be pronounced when $r$ is small.

Our second bound, unlike the Good--Turing estimate, requires no statistical estimation and holds unconditionally, providing a reliable fallback when that estimate may not be trustworthy. It follows directly from Theorem~\ref{thm:prob_of_improvement}. Given $(1/m)^{r-1} \geq 0$,
\[
P\!\left(\phi^{(r)} < \min\!\left(\phi^{(1)}, \ldots, \phi^{(r-1)}
\right)\right) \leq \frac{1}{r}.
\]

Neither of the first two bounds uses information about which specific partition is currently best. Our third, this time empirically motivated, bound incorporates this information by analogy to the argument used in Theorem~\ref{thm:prob_of_improvement}. Treating the currently best-observed partition as if its probability $p^\star$ were fixed and known, and substituting the probability that at least two of $r$ runs land on it, $1 - (1-p^\star)^r - r\,p^\star(1-p^\star)^{r-1}$, in place of the term $(1/m)^{r-1}$ used in that theorem's proof, gives
\[
P\!\left(\phi^{(r)} < \min\!\left(\phi^{(1)}, \ldots, \phi^{(r-1)}\right)\right) \leq \frac{(1-p^\star)^r + r\, p^\star(1-p^\star)^{r-1}}{r}.
\]

This substitution is coherent because the proof of Theorem~\ref{thm:prob_of_improvement} rests on the identity 
\[
r \cdot P(\phi^{(r)} < \min(\phi^{(1)}, \ldots, \phi^{(r-1)})) = 1 - P(\text{at least two runs tie at the minimum}),
\]
and any lower bound on the tie probability yields an upper bound on the probability of improvement. The term $(1/m)^{r-1}$ lower-bounds the tie probability through the event that all $r$ runs agree. The event that at least two runs land on the currently best-observed partition is likewise sufficient for a tie, and since each run lands on it independently with probability $p^\star$ by Lemma~\ref{lemma:iid}, the number of such runs is binomial, giving this event probability 
\[
1 - (1-p^\star)^r - r\,p^\star(1-p^\star)^{r-1}.
\]
We estimate $p^\star$ from the data using a Clopper--Pearson lower confidence bound $p_{lo}$, treating the identity of the best-observed partition as fixed. This is an approximation rather than a rigorous guarantee, since that partition is itself selected from the same runs used to estimate $p^\star$.

Since each of the three quantities upper-bounds $P(\phi^{(r)} < \min(\phi^{(1)}, \ldots, \phi^{(r-1)}))$, so does their
minimum, which is simply the most informative of the three at each run. The algorithm evaluates all three bounds at each run and halts once their minimum falls below $\varepsilon/2$. The tolerance $\varepsilon$ is evenly divided between the two sources of error in the stopping decision. These are the estimated probability that a further restart improves on the best partition found so far (at most $\varepsilon/2$), and the failure of the Clopper--Pearson bound itself, which is computed at confidence level $1 - \varepsilon/2$. Informally, by a union bound, the probability that the algorithm stops without justification is at most $\varepsilon$. For instance, setting $\varepsilon = 0.1$ means accepting at most a $10\%$ chance that a stop was premature. The number of restarts is thus a justified, reportable quantity rather than an arbitrary choice, and $\varepsilon$ a parameter the user can reason about directly.

Since the three bounds are inevitably crude if the number of restarts is too low, we evaluate them after a minimum of $r_{\min}$ restarts. Theorem~\ref{thm:improvement_expectation_lb} indicates how to set this minimum. Taking $r = 2$, the smallest value for which it applies, gives $r' > 2e \approx 5.44$, so no further improvement is expected before restart $6$. We therefore set $r_{\min} = 6$. Algorithm~\ref{alg:gtkmeanspp} formalises the steps.

\begin{algorithm}
\caption{Good--Turing Restart Criterion (GTRC) for $k$-means++}
\label{alg:gtkmeanspp}
\begin{algorithmic}[1]
\Require data set $\mathcal{X}$, cluster count $k$, tolerance $\varepsilon \in (0,1)$.
\Ensure best partition found and its cost $\phi^\star$.
\State Initialise $\phi^\star \gets \infty$, $r \gets 0$, and the multiset $\mathcal{V} \gets \emptyset$.
\Repeat
    \State $r \gets r + 1$.
    \State run $k$-means++ on $\mathcal{X}$ with a fresh random initialisation; let $\phi^{(r)}$ be its cost.
    \State add $\phi^{(r)}$ to $\mathcal{V}$.
    \If{$\phi^{(r)} < \phi^\star$}
        \State $\phi^\star \gets \phi^{(r)}$, and store the corresponding partition.
    \EndIf
    \If{$r \geq 6$}
        \State $\phi_1 \gets$ number of values occurring exactly once in $\mathcal{V}$.
        \State $c^\star \gets$ number of occurrences of the best partition in $\mathcal{V}$.
        \State $p_{lo} \gets$ Clopper--Pearson lower confidence bound on $p^\star$ at level $1-\varepsilon/2$, given $c^\star$ successes in $r$ trials.
        \State $U \gets \min\left(\dfrac{\phi_1}{r},\ \dfrac{1}{r},\ \dfrac{(1-p_{lo})^r + r\,p_{lo}(1-p_{lo})^{r-1}}{r}\right)$.
    \EndIf
\Until{$r \geq 6$ \textbf{and} $U \leq \varepsilon/2$}
\State \Return stored partition and $\phi^\star$.
\end{algorithmic}
\end{algorithm}

\section{Experimental Setup}
\label{sec:setting}

We evaluated the Good--Turing restart criterion (GTRC) on 34 publicly available data sets (for details see Table~\ref{tab:datasets}), downloaded from the popular UCI Machine Learning Repository~\cite{uci-ml-repo}. For each data set, the number of clusters $k$ was set according to its ground truth. Categorical features were one-hot encoded prior to standardisation. Each data set was then standardised feature-wise following
\[
x_{iv} = \frac{x_{iv} - \bar{x}_{v}}{\max x_v - \min x_v},
\]
where $\bar{x}_{v}$ is the average of feature $v$ over all $x_i \in \mathcal{X}$.

For each data set and each tolerance $\varepsilon \in \{0.05, 0.1\}$, we ran GTRC 200 times. For each run, we recorded the final objective, the number of restarts used, which of the three stopping bounds (Good-Turing, unconditional, and confidence) were satisfied at the point of stopping, the number of distinct partitions observed, and $p_{lo}$, the Clopper-Pearson lower confidence bound on $p^*$ at the point of stopping.

We compared our GTRC against four other methods. First, the standard $k$-means++ with a fixed number of restarts in the set $ \{10, 20, 50, 100\}$, we ran 200 times each, covering the range of restart counts commonly used in practice. All runs used the same $k$-means++ implementation, differing only in how the number of restarts was determined. To compare GTRC against each fixed restart count, we used a Wilcoxon signed rank test on the paired objective values across the 200 runs, with objectives rounded to two decimal places. Where the two sets of rounded objectives were identical, no test was performed and the two methods were treated as equivalent.

Second, Hart's rule. For this we set $\beta = \epsilon$, so that the rule's failure tolerance matched that of GTRC. We fixed $\delta = 0.4$, the value used throughout Hart's original experiments \cite{hart1998sequential}. Third, Ohsaki and Yamakawa's rule. We set its stopping threshold equal to GTRC's $\epsilon$. In our experiments, we found this rule can fail to converge on real data. A partition found only once contributes a term to the likelihood ratio that approaches $e^{-1} \approx 0.368$ as $r$ grows, rather than decaying to zero, so data sets that repeatedly produce rarely-recurring partitions can prevent the ratio from ever falling below the threshold. We therefore impose a cap of $1000$ restarts, beyond which the rule is deemed not to have converged, and we record whether each run converged within this cap. Fourth, Corominas's rule. We set $\alpha = \epsilon$, so that its failure tolerance matches the tolerance used throughout. We also fixed $\epsilon_C = 0.01$, the value used in the original experiments \cite{corominas2023deciding}. Since this rule computes a fixed number of restarts in advance rather than adapting to the run in progress, the resulting restart count does not vary across data sets for a given $\epsilon$.

\begin{table}[t]
\centering
\caption{Data sets used in the experiments. $n$ is the number of data points, $d$ the number of features (after one-hot encoding), and $k$ the number of clusters, which we set to the number of ground truth classes.}
\label{tab:datasets}
\footnotesize
{\setlength{\tabcolsep}{4pt}
\begin{tabular}{lrrrlrrr}
\toprule
Data set & $n$ & $d$ & $k$&Data set & $n$ & $d$ & $k$\\
\midrule
Acute Inflammations       &120&11&4&Musk                      &476&166&2\\
Australian Credit         &690&42 &2 &Online News Pop           &39,644&58&6\\
Balance                   &625&20&2&Page Blocks             &5,473&10&5\\
Breast Cancer             &699&9&2&Parkinsons                &195&22&2\\
Car Evaluation            &1728&21&4&Pen Digits                 &10,992&16&2\\
Cover Type                &581,012&54&7&Sensorless Drive           &58,509&48&11\\
Ecoli                     &336&7&8&Skin Segmentation          &245,057&3&2\\
Glass                     &214&9&6&Soya                    &47&72&4\\
HTRU2                     &17,898&8&2&Spambase               &4,601&57&2\\
Hand Postures             &78,095&36&5&Tamilnadu                 &45,781&22&31\\
Heart                     &270&25&2&Teaching Assistant    &151&56&3\\
IDA2016                   &76,000&169&2&Tic Tac Toe               &958&27&2\\
Ionosphere                &351&33&2&Tulugu Vowels              &871&3&6\\
Iris                      &150&4&3&Wall Following Robot  &5,456&24&4\\
Isolet                    &7,797&617&26&Wine                     &178&13&3\\
Landsat                   &6,435&7&6&Wine Quality          &6,497&11&7\\
Letter Recognition        &20,000&16&26&Zoo                      &101&16&7\\

\bottomrule
\end{tabular}
}
\end{table}
\section{Results and discussion}
\label{sec:results}

Here, we evaluated GTRC against fixed-restart $k$-means++, Hart's rule, Ohsaki and Yamakawa's rule, and Corominas's rule across 34 data sets. Choosing a fixed number of restarts for $k$-means++ trades the number of restarts against clustering quality. That is, more restarts tend to produce a lower objective but at a higher computational cost. Figure~\ref{fig:median_gap} summarises this trade-off across the full benchmark using medians rather than means, since the mean gap may be strongly influenced by a small number of data sets. The gap plotted for each method is measured against the best objective observed. At $\varepsilon = 0.05$, GTRC reached a median gap of approximately $0.002\%$ to this reference, with only a median of 32 restarts. This indicates that GTRC identifies a restart count sufficient to reach a near-optimal partition without restarting $k$-means++ more than necessary. At $\varepsilon = 0.1$, GTRC used a median of 19 restarts and reached a gap of approximately $0.009\%$, a restart count and quality trade-off directly interpretable from the chosen tolerance.

\begin{figure}[t]
\centering
\begin{minipage}{0.49\textwidth}
\centering
\includegraphics[width=\textwidth]{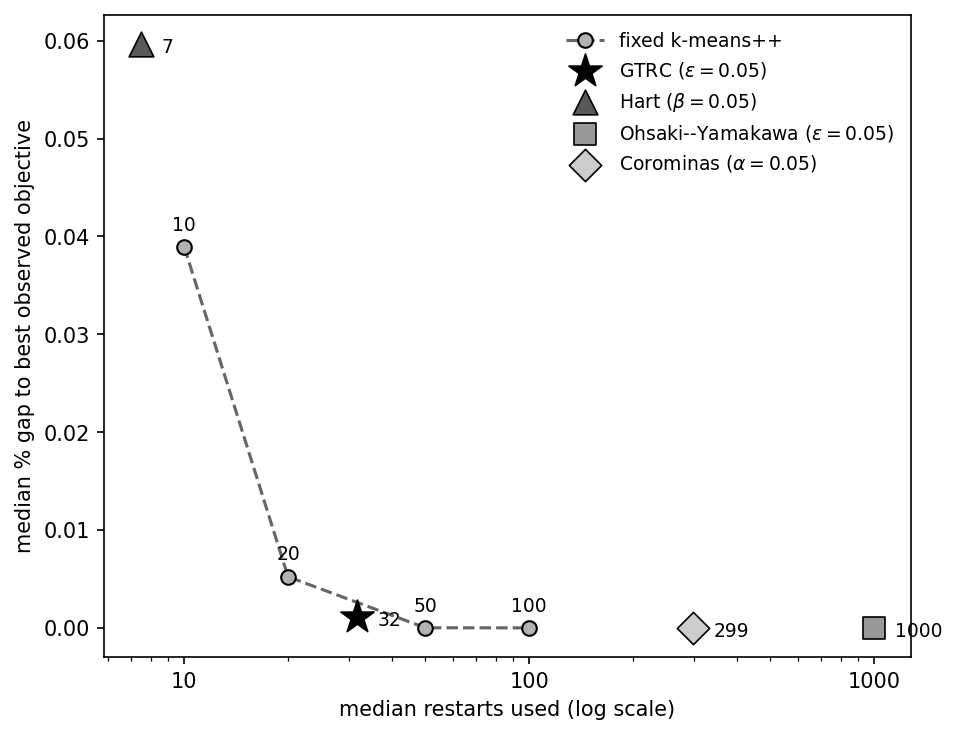}
\end{minipage}
\hfill
\begin{minipage}{0.49\textwidth}
\centering
\includegraphics[width=\textwidth]{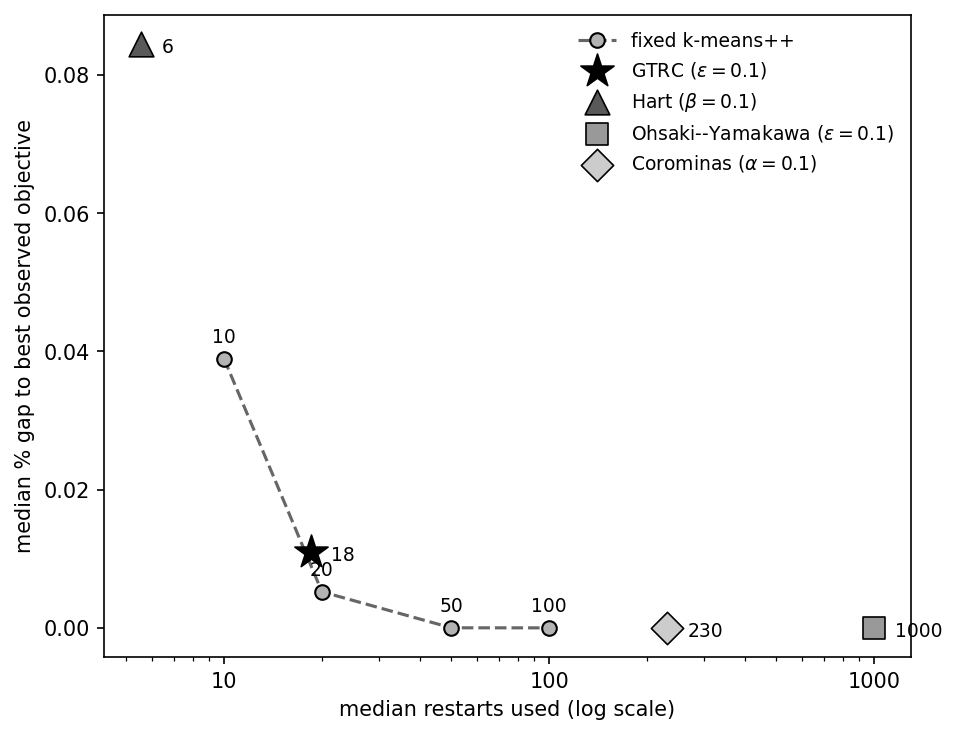}
\end{minipage}
\caption{Median restarts and median \% gap to the best observed objective across 34 data sets, for fixed $k$-means++, GTRC, Hart's rule, Ohsaki and Yamakawa's rule, and Corominas's rule,  each set to a tolerance of $0.05$ (left) and $0.1$ (right).}
\label{fig:median_gap}
\end{figure}

Hart's rule uses far fewer restarts than GTRC, a median of 7 at $\varepsilon = 0.05$ and 6 at $\varepsilon = 0.1$, but at a real cost to clustering quality. GTRC reached a lower gap than Hart's rule on 25 of 34 data sets at $\varepsilon = 0.05$ and on 31 of 34 at $\varepsilon = 0.1$, with Hart's median gap ($0.061\%$ and $0.086\%$ respectively) an order of magnitude larger than that of GTRC. Hart's rule is therefore cheaper, but the restarts it forgoes lead to a considerably worse objective.

Ohsaki and Yamakawa's rule and Corominas's rule reached a median gap indistinguishable from zero, but only by using far more restarts than GTRC. Corominas's rule used a fixed 299 restarts at $\varepsilon = 0.05$ and 230 at $\varepsilon = 0.1$ on every data set regardless of its difficulty, an order of magnitude more than GTRC's median. Ohsaki and Yamakawa's rule used a median of 1,000 restarts at both tolerances, the cap we imposed on it, because it did not satisfy its own stopping criterion within this cap in all 200 repetitions on 19 of the 34 data sets, and in some repetitions on a further 12. This is consistent with the diminishing-returns behaviour established in Section~\ref{sec:theoretical_framework}. That is, once a data set is easy enough that a handful of restarts already reach a near-optimal partition, spending tens or hundreds of times more restarts can close only a negligible remaining gap. GTRC is designed to detect this point directly, rather than to keep searching for the small residual improvement that a much larger restart budget may still find.

Figure~\ref{fig:restart_spread} shows the number of restarts used by GTRC at $\varepsilon = 0.05$, sorted across the 34 data sets. The distribution is far from uniform. Restart counts range from 6, on data sets where a single local minimum is reached reliably, to the ceiling of 40 restarts implied by $\varepsilon = 0.05$ and Theorem~\ref{thm:prob_of_improvement}, reached on 14 data sets. This spread confirms that GTRC does not converge on a single typical restart count. Rather, the number of restarts it uses is set by the difficulty of each data set, which is precisely the information a fixed restart count cannot take into account.

\begin{figure}[t]
\centering
\includegraphics[width=0.6\textwidth]{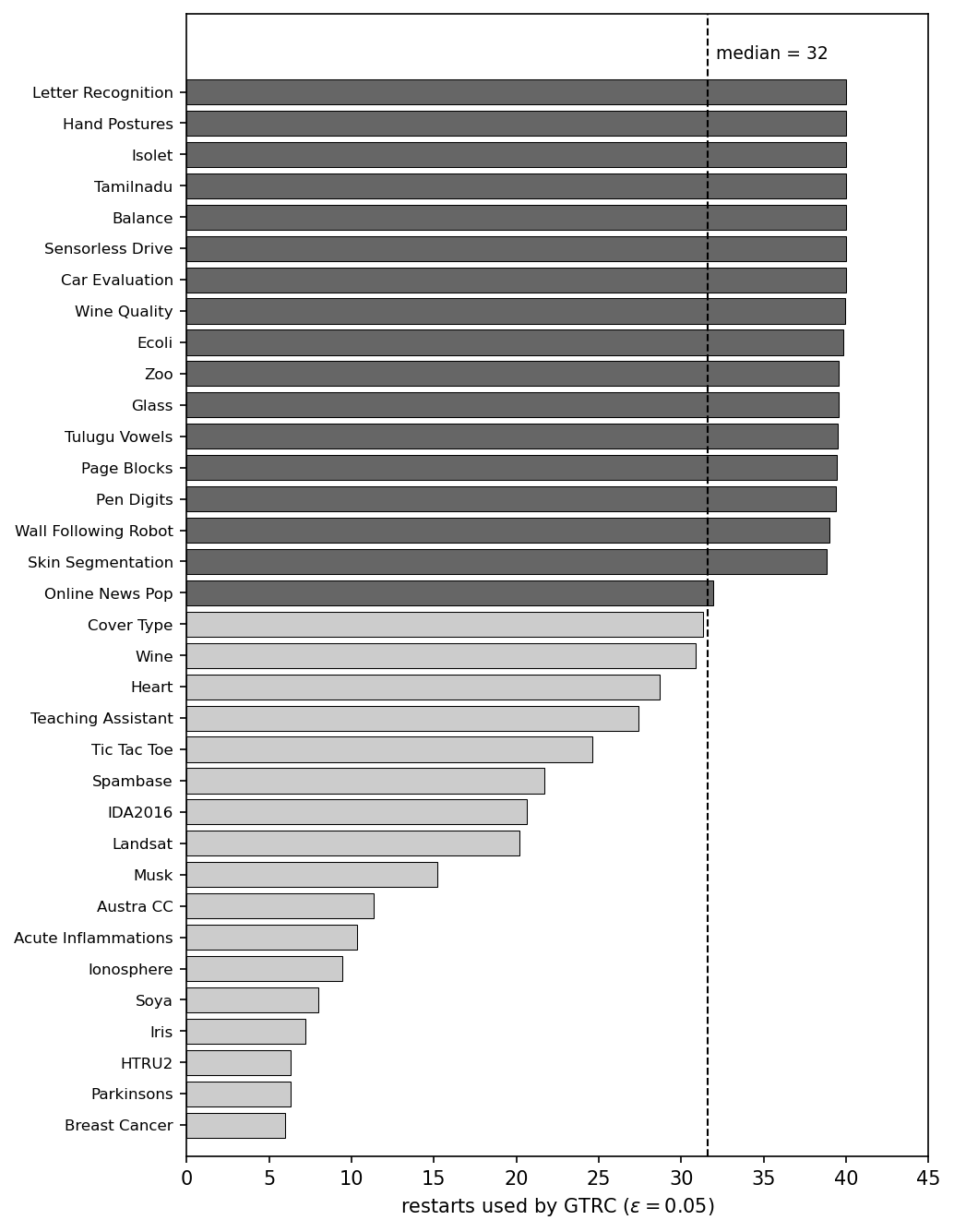}
\caption{Restarts used by GTRC ($\varepsilon = 0.05$) across 34 data sets, sorted, with the median (32) marked.}
\label{fig:restart_spread}
\end{figure}

Figure~\ref{fig:p_lo_vs_runs} shows how the number of restarts used by GTRC relates to $p_{lo}$, the lower confidence bound on $p^*$ (the true probability that a single restart of $k$-means++ converges to the partition currently regarded as best). Across the 34 data sets, the two quantities are almost perfectly monotonically related, with a Spearman correlation of
$\rho = -0.99$ at $\varepsilon = 0.05$ and $\rho = -0.98$ at $\varepsilon = 0.1$. Data sets for which $p_{lo}$ is high are stopped after only a handful of restarts, while data sets for which $p_{lo}$ remains low are run until close to the ceiling on restarts implied by $\varepsilon$. This further confirms that the restart count chosen by GTRC is not arbitrary, but governed by how quickly the algorithm gains confidence in its own result.

\begin{figure}[t]
\centering
\begin{minipage}{0.48\textwidth}
\centering
\includegraphics[width=\textwidth]{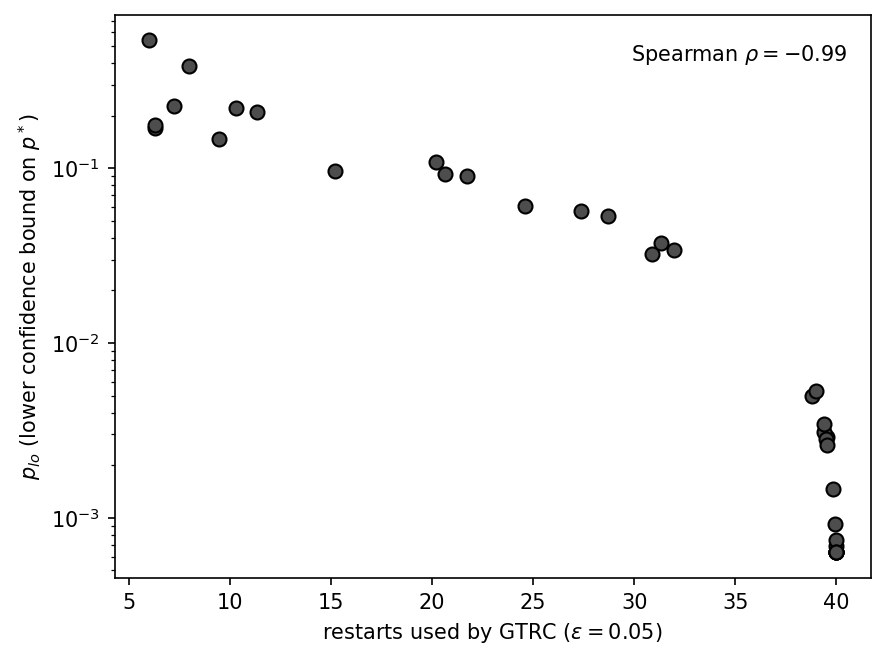}
\end{minipage}
\hfill
\begin{minipage}{0.48\textwidth}
\centering
\includegraphics[width=\textwidth]{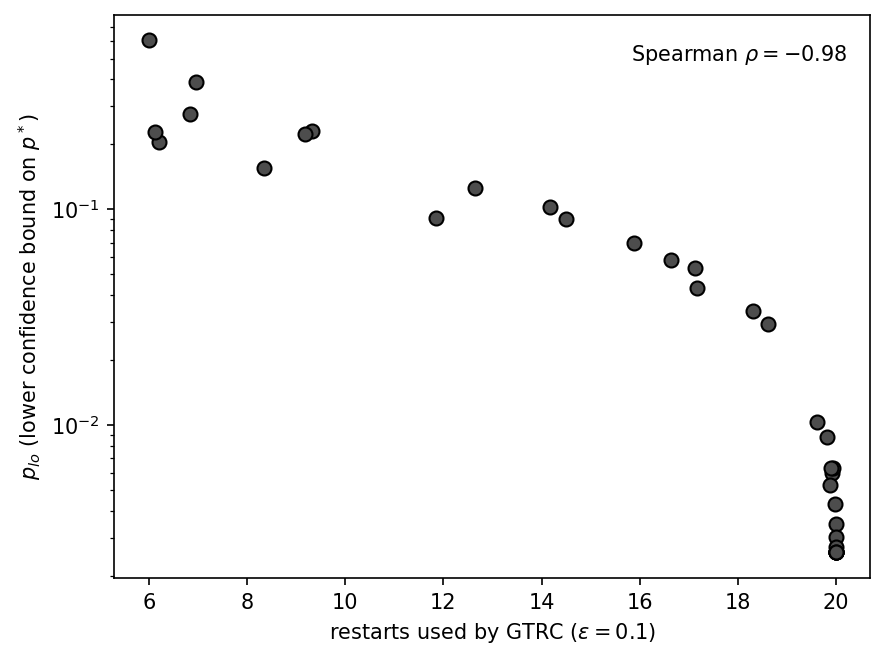}
\end{minipage}
\caption{Restarts used by GTRC against $p_{lo}$, a lower bound on how likely a single restart is to reach the best partition found so far. GTRC stops earlier when this probability is high, confirming that its stopping decision tracks how quickly it becomes confident in its result.}
\label{fig:p_lo_vs_runs}
\end{figure}

Figure~\ref{fig:bound_rates} reports how often each of the three bounds was satisfied ($\leq \varepsilon/2$) at the point GTRC stopped, averaged across the 34 data sets. The confidence bound was satisfied most often (86--87\% on average), followed by the unconditional bound (43--57\%) and the Good-Turing bound (18--19\%). The confidence bound is algebraically bounded below by the unconditional bound, so in this benchmark it is always at least as tight, and the two are frequently satisfied together on the harder data sets. Despite this, the unconditional bound is retained in the stopping rule. Unlike the confidence bound, which relies on an estimated quantity, the unconditional bound holds regardless of how that estimate behaves, and so continues to guarantee termination even in cases where the confidence bound might not be reliable.

\begin{figure}[t]
\centering
\includegraphics[width=0.6\textwidth]{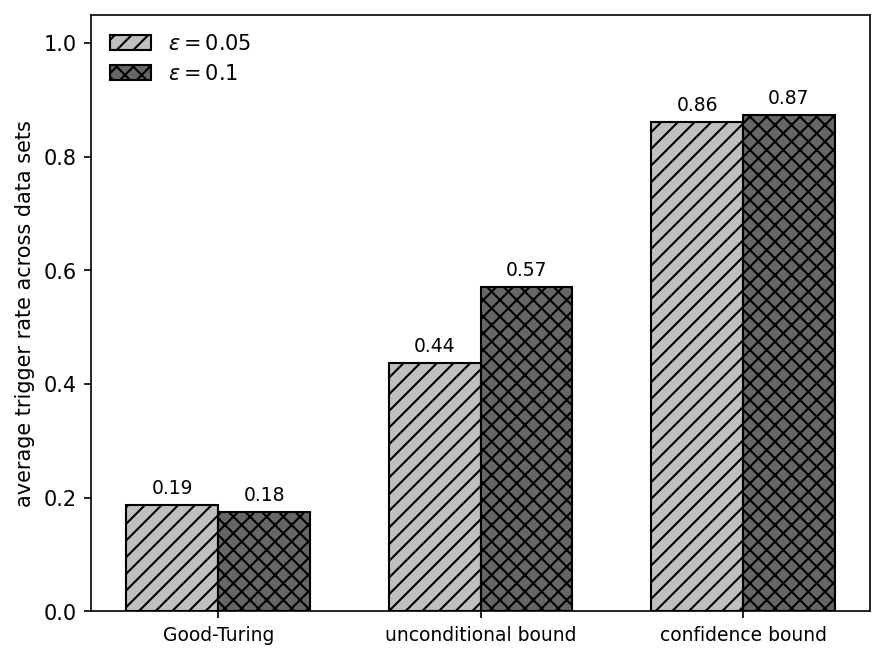}
\caption{Average rate at which each stopping bound was itself satisfied at the point GTRC stopped, across 34 data sets.}
\label{fig:bound_rates}
\end{figure}

Taken together, these results support the central claim of this paper. GTRC reaches clustering quality competitive with well-chosen fixed restart counts, without requiring that count to be specified in advance for each data set. This is not the result of a single restart count that happens to work well on average. Rather, it follows from a stopping decision that responds to how quickly the algorithm becomes confident in its own result, a signal that varies substantially across data sets and that a fixed restart count cannot account for. 

The comparison against Hart's rule, Ohsaki and Yamakawa's rule, and Corominas's rule places this result in context. Hart's rule uses fewer restarts than GTRC, but at a real cost to clustering quality on most data sets. Ohsaki and Yamakawa's rule and Corominas's rule reach a marginally lower gap than GTRC, but do so using an order of magnitude more restarts, a fixed count in the case of Corominas's rule regardless of data set difficulty, and a count large enough that Ohsaki and Yamakawa's rule frequently failed to satisfy its own stopping criterion within it. None of the three alternatives considered adapts its restart count to data set difficulty in the way GTRC does, and none offers the same balance between clustering quality and restart cost across data sets of varying difficulty.

\section{Conclusion}

\label{sec:conclusion}

This paper introduces GTRC, a stopping criterion for $k$-means++ restarts that determines, from the data itself, how many restarts are needed rather than requiring this number to be fixed in advance. Unlike the common practice of fixing a restart count by convention and applying it uniformly regardless of data set difficulty, GTRC combines three bounds on the probability that a further restart would improve on the current result, and stops once this probability falls below a user-specified tolerance $\varepsilon$. Our experiments on 34 real-world data sets show that GTRC identifies the point beyond which further $k$-means++ restarts yield only negligible improvement in terms of objective. This restart count was shown to be governed by $p_{lo}$, a measure of how quickly the algorithm becomes confident in its own result, rather than by an arbitrary or fixed rule.

We further compared GTRC against three existing stopping rules for multistart local search. Hart's rule used fewer restarts than GTRC but at a real cost to clustering quality on most data sets. Ohsaki and Yamakawa's rule and Corominas's rule reached a marginally lower gap than GTRC, but only by using an order of magnitude more restarts, a fixed count regardless of data set difficulty in the case of Corominas's rule, and a count large enough that Ohsaki and Yamakawa's rule frequently failed to satisfy its own stopping criterion within it. None of the three alternatives adapts its restart count to data set difficulty in the way GTRC does.

More broadly, this work addresses a gap in how clustering restarts are reported and compared. Rather than treating the number of restarts as an arbitrary hyperparameter, GTRC offers a principled, interpretable, and data-dependent alternative, with $\varepsilon$ as the only parameter a user needs to set and reason about.

\bibliography{references}

\end{document}